\documentclass[letterpaper, 10 pt, conference]{ieeeconf}  

\IEEEoverridecommandlockouts                              

\usepackage{graphics} 
\usepackage{mathptmx} 
\usepackage{amsmath} 
\usepackage{amssymb}  
\usepackage{amsfonts}
\usepackage{algorithm}
\usepackage{algpseudocode}
\usepackage[dvipsnames]{xcolor}
\usepackage[flushleft]{threeparttable}
\usepackage{booktabs,caption}
\usepackage{array}
\makeatletter
\newcommand{\thickhline}{%
    \noalign {\ifnum 0=`}\fi \hrule height 0.5pt
    \futurelet \reserved@a \@xhline
}

\newtheorem{theorem}{Theorem}
\usepackage{graphicx}
\usepackage{multirow}
\usepackage{threeparttable}
\usepackage{pifont}
\usepackage{cleveref}
\usepackage{newtxtext, newtxmath}

\usepackage{xcolor,colortbl}
\definecolor{LightCyan}{rgb}{0.88,1,1}
\newcommand{\cmark}{\ding{51}}
\newcommand{\xmark}{\ding{55}}
\newcommand{\swarehouse}{$\textsc{small}$}
\newcommand{\lwarehouse}{$\textsc{large}$}
\newcommand{\mwarehouse}{$\textsc{medium}$}
\title{\LARGE \bf
Multi-Goal Multi-Agent Pickup and Delivery*}

\author{Qinghong Xu,$^{1}$ Jiaoyang Li,$^{2}$ Sven Koenig$^{2}$ and Hang Ma$^{1}$
\thanks{*The research at Simon Fraser University was supported
by the Natural Sciences and Engineering Research Council
of Canada under grant number RGPIN2020-06540 as well as a Canada Foundation for Innovation John R. Evans Leaders Fund award. The research at the University of Southern California was supported by the National Science Foundation under grant numbers 1409987, 1724392, 1817189, 1837779, 1935712, 2121028, and 2112533 as well as a gift from Amazon Robotics.}
\thanks{$^{1}$Qinghong Xu and Hang Ma are with the School of Computing Science, Simon Fraser University, Burnaby, BC V5A1S6, Canada
        {\tt\small \{qxa8, hangma\}@sfu.ca}}%
\thanks{$^{2}$Jiaoyang Li and Sven Koenig are with the Department of Computer Science, University of Southern California, Los Angeles, CA 90007, USA
        {\tt\small \{jiaoyanl, skoenig\}@usc.edu}}%
}

\begin{document}
\maketitle
\thispagestyle{empty}
\pagestyle{empty}

\begin{abstract}
In this work, we consider the Multi-Agent Pickup-and-Delivery (MAPD) problem, where agents constantly engage with new tasks and need to plan collision-free paths to execute them. To execute a task, an agent needs to visit a pair of goal locations, consisting of a pickup location and a delivery location. We propose two variants of an algorithm that assigns a sequence of tasks to each agent using the anytime algorithm Large Neighborhood Search (LNS) and plans paths using the Multi-Agent Path Finding (MAPF) algorithm Priority-Based Search (PBS). LNS-PBS is complete for well-formed MAPD instances, a realistic subclass of MAPD instances, and empirically more effective than the existing complete MAPD algorithm CENTRAL. LNS-wPBS provides no completeness guarantee but is empirically more efficient and stable than LNS-PBS. It scales to thousands of agents and thousands of tasks in a large warehouse and is empirically more effective than the existing scalable MAPD algorithm HBH+MLA*. LNS-PBS and LNS-wPBS also apply to a more general variant of MAPD, namely the Multi-Goal MAPD (MG-MAPD) problem, where tasks can have different numbers of goal locations. 
\end{abstract}
\section{INTRODUCTION }
In many real-world multi-robot systems, robots have to constantly attend to new tasks and plan collision-free paths to execute them. For example, warehouse robots need to move inventory shelves to workstations, where human workers pick products from the shelves to fulfill the orders of customers. This problem has been studied as Multi-Agent Pickup-and-Delivery (MAPD) \cite{MaAAMAS17}. In MAPD, each task has a release time and a sequence of two goal locations, namely a pickup location and a delivery location. For a warehouse robot, the pickup location is the storage location of an inventory shelf in the warehouse, and the delivery location  is the location of the workstation that needs a product stored on the inventory shelf. To execute a task, an agent (i.e., robot) needs to first visit its pickup location at or after its release time and then visit its delivery location. 

To solve a MAPD instance, the agents need to decide which tasks they are going to execute and plan collision-free paths to execute them effectively. Most existing MAPD algorithms separate the task-assignment and path-finding parts, i.e., they first assign tasks to the agents based on an estimation of the actual path costs and then use a Multi-Agent Path Finding (MAPF) \cite{SternSOCS19} algorithm for planning actual collision-free paths for the agents. Such decoupled MAPD algorithms can be categorized into (1) those that assign only one task to each agent at a time and plan paths for the agents segment by segment \cite{MaAAMAS17}, i.e., each call to the path planner computes a plan that moves the agents only from their current locations to their next goal locations; (2) assign only one task to each agent but plan paths that move the agents from their current locations through a sequence of goal locations \cite{GrenouilleauICAPS21}; and (3) assign a sequence of tasks to each agent and plan paths for the agents segment by segment \cite{LiuAAMAS19, ChenIEEE21}.  Assigning only one task to each agent can lead to bad task assignments since it does not take the subsequent tasks into account, and planning paths segment by segment can lead to long paths.

In addition, there is some work that focuses only on the path-finding part of the MAPD problem. For instance, Surynek \cite{SurynekAAAI21} proposes the optimal Multi-Goal MAPF algorithms HCBS and SMT-HCBS for planning collision-free paths for a set of goal locations (where the ordering of the goal locations is not specified). However, these algorithms solve only one-shot problems where each agent has only one task, and their scalability is limited. Li et al. \cite{LiAAAI21b} propose the efficient lifelong MAPF algorithm Rolling-Horizon Collision Resolution (RHCR) for planning collision-free paths for a sequence of goal locations. It uses a rolling-horizon framework that repeatedly calls a windowed MAPF algorithm to resolve collisions for only a few timesteps ahead. Such windowed MAPF algorithms run significantly faster than regular MAPF algorithms but typically do not provide a completeness guarantee since they can lead to deadlocks due to their shortsightedness. 
\begin{table}[t]
\caption{Research related to MAPD. ``Lifelong'' means that agents can constantly engage with new tasks. ``Online'' means that the entire task set is unknown in the beginning, and new tasks can enter the system at any time. ``Assign tasks (seq. task)'' means that an algorithm can assign a task sequence (rather than a single task) to each agent. ``Find paths (seq. goals)'' means that an algorithm can plan a path for a sequence of goal locations (rather than segment by segment) for each agent. ``Complete (well-formed)'' means that an algorithm is complete for well-formed MAPD instances.}
\label{MAPD_solver}
\centering
\Huge
\scalebox{0.3}{%
\begin{tabular}{lccccc}
\thickhline
& \multirow{2}{*}{lifelong}& \multirow{2}{*}{online} &assign tasks &find paths  & complete \\
& & &(seq. tasks) & (seq. goals) & (well-formed)\\
\thickhline 
CENTRAL \cite{MaAAMAS17}&\cmark&\cmark &\xmark  & \xmark& \cmark\\
TA-Hybrid  \cite{LiuAAMAS19} &\cmark&\xmark &\cmark&\xmark & \cmark\\
HBH+MLA* \cite{GrenouilleauICAPS21}&\cmark&\cmark &\xmark&\cmark & \cmark\\
RMCA \cite{ChenIEEE21}&\cmark&\cmark &\cmark&\xmark & \xmark\\
(SMT-)HCBS \cite{SurynekAAAI21}&\xmark&N/A&\xmark&\cmark &\cmark\\
RHCR \cite{LiAAAI21b}&\cmark&\cmark &\xmark&\cmark & \xmark\\
\rowcolor{LightCyan}LNS-PBS &\cmark&\cmark &\cmark&\cmark &\cmark\\
\rowcolor{LightCyan}LNS-wPBS &\cmark&\cmark &\cmark&\cmark &\xmark\\
\thickhline
\end{tabular}
}
\end{table}

The main contributions of our work are as follows: We propose a decoupled algorithm that assigns a sequence of tasks to each agent using the anytime algorithm Large Neighborhood Search (LNS) and plans paths through a sequence of goal locations using the MAPF algorithm Priority-Based Search (PBS). More specifically, we propose two variants of this algorithm: LNS-PBS and LNS-wPBS. The first variant focuses on completeness and effectiveness. PBS is, in general, incomplete. Combined with the idea of ``reserving dummy paths" from \cite{LiuAAMAS19}, LNS-PBS is complete on well-formed MAPD instances, a realistic subclass of MAPD instances. The second variant focuses on  efficiency and stability. LNS-wPBS uses the windowed MAPF algorithm of RHCR. Therefore, the runtime of LNS-wPBS is controlled by the user-specified runtime limit for the anytime task-assignment algorithm and the user-specified size of the time window for the windowed MAPF algorithm. Empirically, LNS-PBS and LNS-wPBS often yield smaller service times than state-of-the-art MAPD algorithms, and LNS-wPBS  scales to thousands of agents and thousands of tasks in a large warehouse.

As a further contribution, we study two extensions of the MAPD problem. First, LNS-PBS and LNS-wPBS can extend to a more general variant of the MAPD problem, namely the Multi-Goal MAPD (MG-MAPD) problem, where tasks have different numbers of goal locations. This problem models the scenario where a warehouse robot may need to deliver an inventory shelf to multiple workstations because they all have requested products stored on the same inventory shelf. We prove that LNS-PBS is complete for well-formed MG-MAPD instances. Second, LNS-PBS and LNS-wPBS can handle different MAPD settings. This includes the online setting \cite{MaAAMAS17}, where the entire task set is unknown in the beginning and new tasks can enter the system at any time, the offline setting \cite{LiuAAMAS19}, where the entire task set is known in the beginning, and the semi-online setting (which has not been studied before), where the entire task set is (only) partially known in the beginning. We compare existing MAPD-related algorithms against our algorithms LNS-PBS and LNS-wPBS in  \Cref{MAPD_solver}.

\section{RELATED WORK}
A MAPD algorithm consists of two components: task assignment and path finding. In this section, we discuss existing research that relates to them.

\subsection{Task Assignment}
The task-assignment problem is related to the multi-robot task allocation literature. Gerkey et al. \cite{GerkeyIJRR04} and  Korsah et al. \cite{KorsahIJRR13} provide taxonomies for this topic. The Hungarian algorithm \cite{Kuhn55thehungarian} is a combinatorial optimization algorithm that finds the maximum-weight matching in a bipartite graph in polynomial time. Other related problems include the Traveling Salesman Problem (TSP), Vehicle Routing Problem (VRP), and  Dial-a-Ride Problem. Shaw \cite{Shaw97anew} introduces the local search algorithm Large Neighborhood Search (LNS) to construct a customer schedule for the VRP. The idea is to start with an initial schedule and iteratively improve it. In every iteration, some customers are removed from the schedule based on a removal heuristic. These customers are then inserted back into the schedule (at potentially different positions) by a greedy heuristic. 

\subsection{Path Finding}
The path-finding problem is related to the Multi-Agent Path Finding (MAPF) literature. Many MAPF algorithms exist, such as the complete and optimal MAPF algorithm Conflict-Based Search (CBS) \cite{SharonAI15}, its improved variant Improved CBS (ICBS) \cite{BoyarskiIJCAI15}, and the incomplete and suboptimal MAPF algorithm prioritized planning \cite{ErdmannAl87}. 
Given a total priority ordering of the agents, prioritized planning computes the time-minimal paths of the agents in order of their priorities such that the path of an agent does not collide with the paths of all higher-priority agents. Prioritized planning is very efficient, but a pre-defined total priority ordering can make prioritized planning ineffective and even incomplete for hard MAPF instances. Priority-Based Search (PBS) \cite{MaAAAI19a} attempts to address this issue by using depth-first search to find a good total priority ordering. Nevertheless, prioritized planning is faster than PBS for easy MAPF instances in general.

Multi-Label A* (MLA*) \cite{GrenouilleauICAPS21} was invented for planning paths for pairs of goal locations, namely the pickup location and the delivery location of a task. Li et al. \cite{LiAAAI21b} generalize MLA* for planning paths for longer sequences of goal locations. 

\subsection{Combined Task Assignment and Path Finding}
Ma et al. \cite{MaAAMAS17} present the complete MAPD algorithm CENTRAL for well-formed MAPD instances, a realistic subclass of MAPD instances. CENTRAL uses the Hungarian algorithm \cite{Kuhn55thehungarian} to assign each agent one task and then uses CBS for planning collision-free paths for the agents segment by segment that visit the goal locations of their assigned tasks. CENTRAL is designed for online MAPD, where tasks can enter the system at any time. TA-Hybrid \cite{LiuAAMAS19} is designed for the offline setting, where all tasks are known in the beginning. TA-Hybrid  formulates the task-assignment problem as a TSP and uses the anytime TSP algorithm LKH3 \cite{Helsgaun2017} to find a task sequence for each agent. It then uses ICBS for planning collision-free paths for the agents segment by segment that visit the goal locations of their assigned tasks.  Grenouilleau et al. \cite{GrenouilleauICAPS21} propose an H-value-Based Heuristic (HBH) to assign an agent its next task greedily and then prioritized planning and MLA* for planning collision-free paths for the agents that visit the pairs of goal locations of their assigned tasks.  

The MAPD algorithms above are decoupled, i.e., they first assign tasks to the agents based on an estimation of the actual path costs and then use a MAPF algorithm for planning actual collision-free paths for the agents. Chen et al. \cite{ChenIEEE21} propose the coupled MAPD algorithm RMCA, that assigns tasks and plans paths simultaneously. Therefore, its task assignment is informed by the actual path costs. For the task-assignment part, RMCA uses LNS to compute a task sequence for each agent. It first uses a standard regret-based marginal-cost heuristic to construct an initial solution. It then iteratively removes and reassigns a subset of tasks based on a greedy heuristic. For the path-finding part, it uses prioritized planning with sequential A* calls for planning collision-free paths for the agents.

\section{PROBLEM DEFINITION}
In this section, we formalize a generalization of the MAPD problem, namely the Multi-Goal MAPD (MG-MAPD) problem.  MAPD is a special case of MG-MAPD with only two goal locations for each task. A MG-MAPD instance consists of a set of $M$ agents $\{a_1, a_2, ..., a_M\}$ and an undirected graph $G=(V, E)$, whose vertices $V$ represent the set of locations and whose edges $E$ represent the connections between locations that the agents can move along. Let $p_i(t)$ denote the location of agent $a_i$ at timestep $t$. Agent $a_i$ starts at its start location $p_i(0)$; at each timestep, it either moves to an adjacent location or waits at its current location.  A vertex collision occurs between agents $a_i$ and $a_j$ at timestep $t$ iff $p_i(t)=p_j(t)$; an edge collision occurs iff $p_i(t)=p_j(t+1)$ and $p_i(t+1)=p_j(t)$. 

At each timestep, the system can release new tasks. Each task $\tau_i$ is characterized by a sequence of goal locations and a finite release time $r_i \in \mathbb{N}$; we let $s_i$ denote its first goal location and $g_i$ denote its last goal location.  To execute $\tau_i$, an agent needs to visit all goal locations of $\tau_i$ in sequence.  When an agent arrives at $s_i$, it starts to execute $\tau_i$ at or after timestep $r_i$ and cannot execute other tasks; the completion time of $\tau_i$ is the time when the agent arrives at $g_i$. Agents that are assigned tasks are called task agents; otherwise, they are called free agents. 

Not all MG-MAPD instances are solvable; in this work, we consider well-formed MG-MAPD instances, a realistic subclass of MG-MAPD instances \cite{MaAAMAS17, LiuAAMAS19}. We define two types of endpoints: (1) all goal locations of tasks are called task endpoints,  and (2) all start locations of agents are called non-task endpoints. A MG-MAPD instance is well-formed iff the start location of each agent is different from all task endpoints and, for any two endpoints, there exists a path between them that traverses no other endpoints. 

The problem of MG-MAPD is to assign tasks to agents and plan collision-free paths for the agents to execute all tasks assigned to them. The effectiveness of a MG-MAPD algorithm is measured by the average service time. The service time of a task is the difference between its completion time and its release time, i.e., the time that the task spends in the system. The efficiency is measured by the average runtime per timestep. We say that a MG-MAPD algorithm is stable iff its runtime at different timesteps is controllable or predictable.   

\section{LNS-PBS AND LNS-wPBS}
In LNS-PBS and LNS-wPBS, each agent maintains 
(1) a dummy endpoint, i.e., an endpoint that it can move to and stay indefinitely at without collisions (initially, this dummy endpoint is its start location), 
(2) a task sequence, that consists of the uncompleted tasks that it has to execute,
(3) a corresponding goal sequence, that consists of all goal locations of the tasks in its task sequence plus its dummy endpoint at the end, 
and (4) a path, that moves the agent from its current location through all locations in its goal sequence without collisions. 
\Cref{framework_code} without the blue parts (i.e., Lines [6-8]) shows how LNS-PBS works. Many of its steps (not shown in the pseudo-code but introduced later), including the use of dummy endpoints, the strategy of which unexecuted tasks can be assigned to agents, and the modification of PBS, are designed to ensure its completeness.
When new tasks are released by the system, tasks are deferred from the previous iteration,
or a task agent becomes a free agent [Line 2], we start a new iteration and update the four items maintained by the agents: 
First, we use LNS to (re)assign agents those unexecuted tasks, denoted by $\mathcal{T}$, all of whose goal locations are different from the dummy endpoints of the agents [Line 3]. (The other unexecuted tasks are deferred to the next iteration and assigned then.) This destroys the current task sequences of all agents (except for the tasks they are currently executing) and replans new task sequences for them. Then, we assign each agent a (potentially new) dummy endpoint [Line 4]  and use PBS to (re)plan their paths [Line 5]. We will explain Lines [3], [4], and [5] in Sections \ref{TA}, \ref{sec:dummy}, and \ref{PF}, respectively. 
We will prove the completeness of LNS-PBS for well-formed MG-MAPD instances in Section \ref{sec:completeness} and finally introduce LNS-wPBS (i.e., Lines [6-8]) in Section \ref{sec:LNS-wPBS}.
\begin{algorithm}[t]
\caption{LNS-\textcolor{blue}{w}PBS}
\begin{algorithmic}[1]
\While{true}
\If{there are new or deferred tasks

\textbf{or} any task agent becomes a free agent}
\State (Re)assign tasks in $\mathcal{T}$ to agents using LNS;
\State Assign a dummy endpoint to each agent;
\State Plan paths for all agents using {\color{blue}w}PBS;
{\color{blue}\ElsIf{agents have moved $w$ timesteps}
\State Assign a dummy endpoint to each agent;
\State Plan paths for all agents using wPBS;}
\EndIf
\State Agents follow their paths for one timestep;
\EndWhile
\end{algorithmic}
\label{framework_code}
\end{algorithm}

\subsection{Large Neighborhood Search (LNS)} \label{TA}

LNS starts with an initial task assignment generated by Hungarian-based insertion also (introduced below) and iteratively improves it using Shaw removal and regret-based re-insertion (introduced below) until a user-specified runtime limit is reached. In each iteration, LNS accepts the new task assignment if it yields a smaller estimated service time than the old task assignment. In this section, 
we use the term ``estimated'' time to indicate the time calculated under the assumption that all agents follow their shortest paths on graph $G$ that ignore the collisions between each other.
To avoid having to plan a path for a long task sequence on Line 5, we truncate the task sequence of each agent to a size of at most the user-specified maximum size $C$. The remaining tasks are deleted from the task sequences and will be assigned in future iterations, for example, when a task agent completes its current $C$ tasks and becomes a free agent.

\textbf{Hungarian-Based Insertion.} We use the Hungarian algorithm \cite{Kuhn55thehungarian} to construct the initial task assignment, i.e., the task sequences of all agents. Each call to the Hungarian algorithm adds one task to the end of the task sequence of each agent. We repeatedly call it until all tasks in $\mathcal{T}$ have been assigned to agents. In each call, the Hungarian algorithm takes a cost matrix as input (whose rows correspond to agents and whose columns correspond to tasks) and outputs an agent-task assignment with the minimum sum of costs. Previous work \cite{MaAAMAS17} defines an element of the cost matrix as the estimated time for an agent to move from its current location to the first goal location of a task. This choice prioritizes those tasks whose first goal locations are near the current locations of the agents, without considering the release and completion times of the tasks.  Instead, we define an element of the cost matrix as the estimated completion time of a task $\tau_i$ executed by agent $a_i$ under the assumption that $\tau_i$ is inserted at the end of the task sequence of $a_i$.

\textbf{Shaw Removal.} After the construction of the initial task sequences of all agents, we use a Shaw removal operator \cite{RopkeTS06} to remove a group of interrelated tasks from the task sequences. We let $d(u, v)$ represent the shortest-path distance from location $u \in V$ to location $v \in V$. We define the relatedness of two tasks $\tau_i$ and $\tau_j$ as
\begin{align*} 
r(\tau_i, \tau_j) &= \omega_1(d(g_i, g_j)+d(s_i, s_j))\\
&+\omega_2(|t(s_i)-t(s_j)|+|t(g_i)-t(g_j)|),
\end{align*}%
where $t(s_i)$ represents the estimated time when an agent starts to execute  $\tau_i$ (i.e., when the agent reaches the first goal location $s_i$ of $\tau_i$) and $t(g_i)$ represents the estimated completion time of  $\tau_i$. The first term expresses the spatial relatedness of the tasks, and the second term expresses their temporal relatedness. The spatial relatedness and temporal relatedness are weighted by $\omega_1$ and $\omega_2$, respectively. The Shaw removal operator works as follows: We first choose a task $\tau^*$ randomly. We then remove $\tau^*$ and a group of $N-1$ tasks in decreasing order of their relatedness to $\tau^*$, where the neighborhood size $N$ is a user-specified parameter. We also tested other removal operators (e.g., the removal of random tasks and the removal of ``bad'' tasks \cite{RopkeTS06}) in our experiments, but the above removal operator outperformed the others. 

\textbf{Regret-Based Re-Insertion.} We then use a re-insertion operator to re-insert the removed tasks into the task sequences. Specifically, we use the regret-based operator from \cite{ChenIEEE21}\cite{RopkeTS06}.
Let $f_{i}(k, j)$ denote the estimated total service time of the task sequences obtained when inserting task $\tau_i$ at the $j$th position of the task sequence of agent $a_k$ (here, the task sequences do not contain the other removed tasks). Let  $f_{i}^{(1)}$ denote the estimated total service time of the task sequences obtained when inserting task $\tau_i$ at its best position, namely the one with the smallest estimated total service time, i.e., $f_{i}^{(1)}=\min \{f_{i}(k, j) \mid k\in\{1,...,M\}, j\in\{0, ..., l_k\}\}$, where $l_k$ is the number of tasks in the task sequence of $a_k$. Let $f_{i}^{(2)}$ denote the estimated total service time of the task sequences obtained when inserting task $\tau_i$ at its second-best position, namely the one with the second-smallest estimated total service time. The regret of a task $\tau_i$ is defined as $f_{i}^{(2)}-f_{i}^{(1)}$, i.e.,  the difference in the estimated total service time of inserting $\tau_i$ at its best two positions. 
The regret-based re-insertion operator works as follows: We choose the task with the maximum regret, insert it at its best position, and update the regret of the remaining tasks based on the resulting task sequences. We repeat the process until all removed tasks have been re-inserted into the task sequences.

\subsection{Dummy-Endpoint Assignment} \label{sec:dummy}
We assign dummy endpoints one by one with task agents first and free agents afterward. The dummy endpoint of each agent needs to be different from the already assigned dummy endpoints, all goal locations of the uncompleted tasks, and the old dummy endpoints of the other $M-1$ agents in the previous iteration.
When choosing a dummy endpoint for an agent, we consider the task endpoints in increasing order of their shortest-path distances to the last goal location of the last task of the agent. If there are no available task endpoints to assign, we use its start location as its dummy endpoint instead. 

\subsection{Priority-Based Search (PBS)}
\label{PF}
PBS \cite{MaAAAI19a} is an incomplete and suboptimal two-level MAPF algorithm. On the high level, PBS builds a priority tree (PT) and performs a depth-first search on it to construct a priority ordering of the agents. PBS starts with the root node, that contains an empty priority ordering and a time-minimal path for each agent that ignores collisions. When resolving a collision between two agents, PBS generates two child nodes and adds an additional priority relation to each of them: It adds to one child node that the first agent involved in the collision has a higher priority than the second one and vice versa for the other child node. On the low level, PBS uses A* for planning time-minimal paths for agents that are consistent with the priority ordering generated by the high level (i.e., lower-priority agents are not allowed to collide with higher-priority agents). PBS prunes the child node iff no such paths exists. Li et al. \cite{LiAAAI21b} generalize the low level of PBS to planning time-minimal paths for agents with sequences of goal locations. 

We modify the low level of PBS to make PBS complete for well-formed MG-MAPD instances.
Before PBS starts, we save the (old) paths computed in the previous iteration. In the first iteration, the old paths are the paths that keep the agents at their start locations indefinitely. These old paths might not visit the goal locations in the current goal sequences but are guaranteed to be collision-free. When PBS generates the root node of the PT, it plans a time-minimal path for each agent that avoids the old paths of all other $M$-1 agents. When PBS resolves a collision between two agents, for each agent whose path needs to be re-planned in each child node, PBS plans a time-minimal path for it that avoids collisions with the new paths of all higher-priority agents and the old paths of all other agents (that do not have a higher priority than it). When a MG-MAPD instance is well-formed, this modification always finds a path for each agent (which we prove below), and thus no PT node is pruned. Since PBS performs depth-first search, the number of expanded PT nodes is no larger than the maximum depth of the PT, which is $\mathcal{O}(M^2)$~\cite{MaAAAI19a}.

\subsection{Completeness of LNS-PBS}\label{sec:completeness}
\begin{theorem}
Given a well-formed MG-MAPD instance with a finite number of tasks, LNS-PBS is guaranteed to find collision-free paths in finite time that allow each agent to execute all tasks assigned to it.
\end{theorem}
\begin{proof}
We first prove that, given a goal sequence for each agent, LNS-PBS (or, more specifically, Line 5 of \Cref{framework_code}) is guaranteed to find collision-free paths in finite time that allow each agent to visit all goal locations in its goal sequence: For the root node, this property holds since such a path exists for each agent. For example, the agent can first follow its old path and stay at its old dummy endpoint until all other $M-1$ agents have completed their old paths and stay at their old dummy endpoints indefinitely. The agent can then visit all goal locations in its goal sequence in order and finally stay at its new dummy endpoint indefinitely without having to pass through the old dummy endpoints of all other $M-1$ agents. This is so since the MG-MAPD instance is well-formed and these dummy endpoints are different from all goal locations in its goal sequence and its new dummy endpoint. Similarly, for each non-root node, this property holds since such a path exists for each agent whose path needs to be replanned. For example, the agent can first follow its old path and stay at its old dummy endpoint until (1) all higher-priority agents have completed their new paths and stay at their new dummy endpoints indefinitely and (2) all other agents have completed their old paths and stay at their old dummy endpoints indefinitely. The agent can then visit all goal locations in its goal sequence in order and finally stay at its new dummy endpoint indefinitely without having to pass through the new dummy endpoints of the higher-priority agents and the old dummy endpoints of all other agents. This is so since the MG-MAPD instance is well-formed and these dummy endpoints are different from all goal locations in its goal sequence and its new dummy endpoint. Therefore, no PT node is pruned, and LNS-PBS is guaranteed to find collision-free paths in finite time that allow each agent to visit all goal locations in its goal sequence.

We then prove that each task is eventually assigned to and completed by some agent. The last task enters the system at some (finite) timestep $t$. We assume for a proof by contradiction that, from then on, all still unexecuted tasks remain unexecuted. Let timestep $t' \ge t$ be the earliest (finite) timestep when all agents are free. At timestep $t'$, no unexecuted task is in $\mathcal{T}$ since, otherwise, at least one task would be assigned to an agent by LNS-PBS and completed by the agent, as argued above. Therefore, at timestep $t'+1$, at least one task was deferred and is added to $\mathcal{T}$. Then, at least one task is assigned to an agent by LNS-PBS and completed by the agent, as argued above, which contradicts the assumption. Therefore, at least one additional unexecuted task is completed by an agent. Applying the argument repeatedly shows that each task is eventually assigned to and completed by some agent. 
\end{proof}

\subsection{LNS-wPBS}\label{sec:LNS-wPBS}
LNS-wPBS is a variant of LNS-PBS that, unlike LNS-PBS, uses windowed PBS (wPBS) for planning collision-free paths for only the first $w$ timesteps and then plan path again once the agents have moved for $w$ timesteps. This makes LNS-wPBS more efficient than LNS-PBS but incomplete because there is no guarantee that the agents can reach their goal locations in a finite number of timesteps. Nevertheless, LNS-wPBS always successfully finds solutions in our experiments. 

Since LNS-wPBS gives up the completeness guarantee, we further simplify it in three respects:
First, LNS-wPBS does not defer any tasks, i.e., $\mathcal{T}$ consists of all unexecuted tasks.
Second, wPBS uses the original low level of PBS instead of our modified version, i.e., it does not consider the old paths of the agents. 
Third, the assigned dummy endpoints need only to be pairwise different from each other without worrying about the goal locations of uncompleted tasks and the old dummy endpoints.

\section{LOOK-AHEAD HORIZONS}
In this section, we study the semi-online setting, where the system has partial knowledge of future tasks and can thus plan for them. In this case, we consider all known tasks when generating task sequences. We divide tasks into batches, where all tasks in one batch are released at the same timestep. We define the look-ahead horizon as the number of batches that we know in advance. For example, if the system releases one task every five timesteps (= 0.2 tasks per timestep), a look-ahead horizon of 1 means that, at timestep 0, we know the tasks that will be released at timestep 0 and 5. In the offline setting, the look-ahead horizon is infinite.  

If the system knows an incoming task ahead of its release time, then we can send an agent to its first goal location and let the agent wait for the task to be released.
\begin{figure}[t]
\centering
\includegraphics[width=0.3\textwidth]{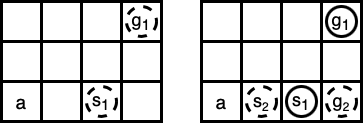}
\caption{A MG-MAPD instance with one agent at timesteps 0 (left figure) and 2 (right figure). Solid circles represent the goal locations of the tasks released before the current timestep and dashed circles represent the ones of the tasks released at the current timestep. ``a'' represents the start location of the agent.}
\label{look_ahead_example}
\end{figure}
For example, in Fig. \ref{look_ahead_example}, the system releases task $\tau_1$ at timestep 0 and task $\tau_2$ at timestep 2. If we have no knowledge of $\tau_2$ at timestep 0, then we assign the agent $\tau_1$ at timestep 0 and then $\tau_2$ at timestep 2. Thus, the completion times of $\tau_1$ and $\tau_2$ are timesteps 5 and 11, respectively, resulting in an average service time of $(5-0+11-2)/2=7$. However, if we use a look-ahead horizon of 1, we can let the agent first move to $s_2$, wait for one timestep, start to execute $\tau_2$ at timestep 2, and start to execute $\tau_1$ at timestep 5. Thus, the completion times of $\tau_2$ and $\tau_1$ are timesteps 4 and 8, respectively, resulting in an average  service time of $(8-0+4-2)/2=5$.

\section{EXPERIMENTS}
We first compare LNS-PBS and LNS-wPBS empirically with the existing MAPD algorithms CENTRAL, RMCA, and HBH+MLA* on MAPD instances in both online and offline settings. We do not compare them with the other three existing MAPD algorithms in TABLE \ref{MAPD_solver} because TA-Hybrid does not work for the online setting, and (SMT-)CBS and RHCR need external algorithms to assign tasks to agents.
We then compare different task-assignment algorithms based on Hungarian and LNS on MG-MAPD instances and finally test LNS-wPBS in a semi-online setting. We use the warehouse environment \swarehouse\ from \cite{LiuAAMAS19}: As shown in Fig. \ref{warehouse_map}, it is a 4-neighbor grid map of size $35 \times 21$  that consists of four columns of endpoints (blue and orange cells) on both left- and right-hand slides of the map and $2 \times 5$ blocks of shelves (horizontal 10-cell-wide strips of black cells) in the middle, each surrounded by task endpoints (blue cells) on the rows above and below it. 
To show the scalability of  LNS-PBS and LNS-wPBS,  we further enlarge \swarehouse\ to \mwarehouse\ and \lwarehouse\, i.e., they are maps of sizes $101 \times 81$ and $187 \times 153$ that contain $8 \times 40$ and $15 \times 76$ strips of shelves, respectively.
We use 500 tasks for \swarehouse, 1,000 tasks for \mwarehouse, and different numbers of tasks for \lwarehouse\ (specified later). The system releases $f$ tasks per timestep. The goal locations of each task are chosen from all task endpoints randomly. The number of goal locations of each task in MG-MAPD instances is chosen from 1 to 5 randomly.
Our experiments are performed on a macOS 2.3 GHz Intel Core i5 with 8 GB RAM. All algorithms are implemented in C++. The implementations of LNS-PBS and LNS-wPBS are based on the RHCR codebase  \cite{LiAAAI21b}. The implementations of the other algorithms are from the original authors. For LNS-PBS and LNS-wPBS, we set the neighborhood size to $N=2$ (we tried $N=$ 2, 4, 8, 32, and 64 for 500 tasks and found that 2 was best) and the runtime limit for LNS to 1s. We use parameters $\omega_1=9$ and $\omega_2=3$ from \cite{RopkeTS06} for the Shaw removal operator. We set the truncated size of the task sequences to $C=2$ (we tried $C=$ 1, 2, 3 for 500 tasks and found that 2 was best). We set the time window of LNS-wPBS to $w=10$ timesteps. We use ``st'' to short for the average service time per task and ``rt'' for the average runtime (ms) per timestep.

\begin{figure}[t]
\centering
\includegraphics[width=0.2\textwidth]{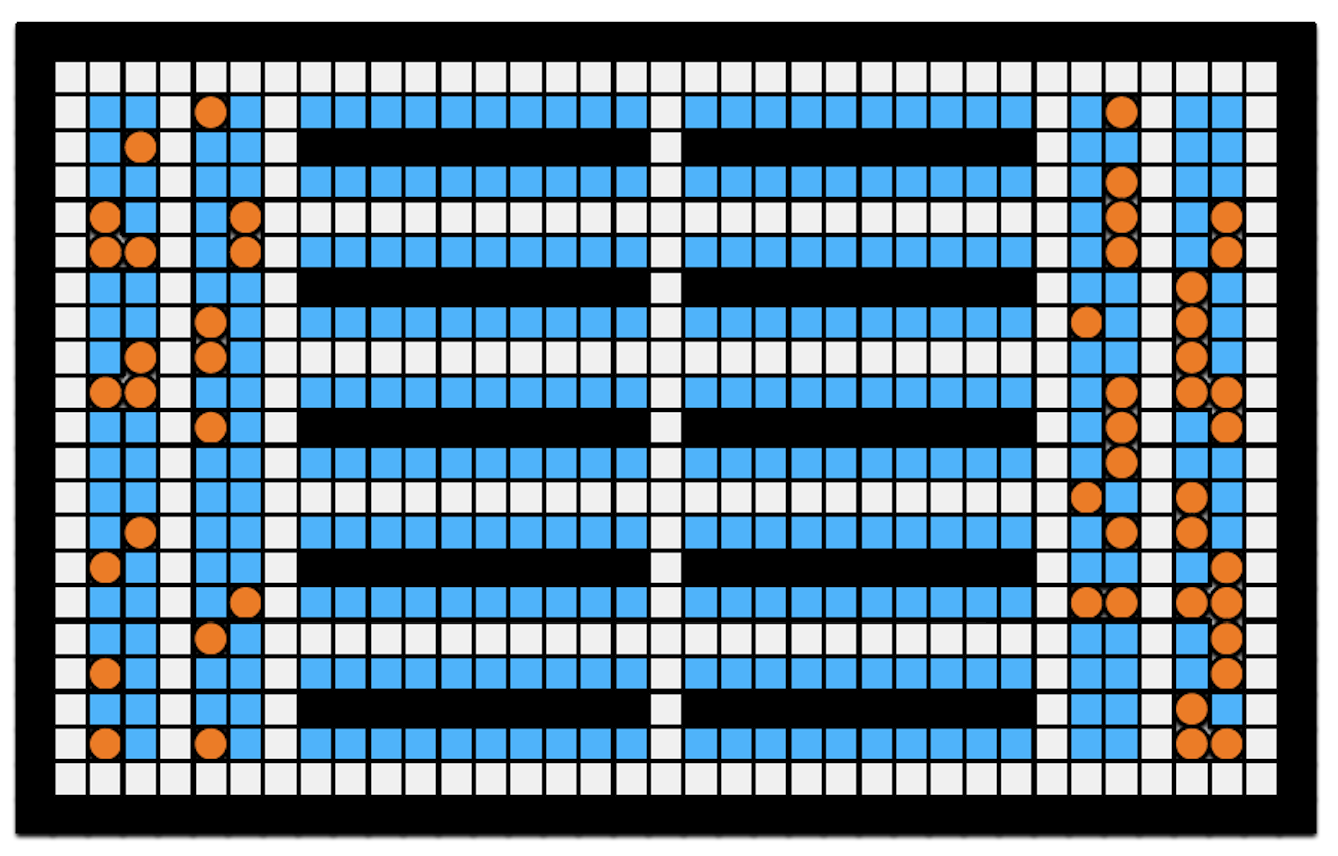}
\caption{Simulated warehouse environment \swarehouse\ from \cite{LiuAAMAS19}. Black cells are blocked. Blue and orange cells represent task endpoints and non-task endpoints, respectively. 
}
\label{warehouse_map}
\end{figure}

\textbf{Small Warehouse.}
TABLE \ref{small_warehouse} compares LNS-PBS and LNS-wPBS with CENTRAL and RMCA on MAPD instances in \swarehouse.
For the complete MAPD algorithms CENTRAL and LNS-PBS, LNS-PBS yields smaller service times than CENTRAL in most cases, with the largest gap being $26\%$ (on the MAPD instance with $f=2$ and $M=50$). It also runs faster than CENTRAL except on very small MAPD instances, i.e., MAPD instances with low task frequencies and small numbers of agents. We do not report the results for the complete MAPD algorithm HBH+MLA* because its service times are worse than those of CENTRAL.
On the other hand, the results for the incomplete algorithms RMCA and LNS-wPBS in the online setting are mixed: neither algorithm dominates the others with respect to efficiency or effectiveness. Nevertheless, LNS-wPBS is more efficient and effective than RMCA in the offline setting.
\begin{table}[t]
\caption{Results on MAPD instances in \swarehouse. ``N/A'' means that the total runtime exceeds 30 minutes. Gap is the average gap measured between the complete MAPD algorithms LNS-PBS and CENTRAL and between the incomplete MAPD algorithms LNS-wPBS and RMCA.}
\centering
\Huge
\scalebox{0.3}{
\begin{tabular}{cc|rr|rr|rr|rr}
\thickhline
& & \multicolumn{2}{c}{CENTRAL} & \multicolumn{2}{|c|}{LNS-PBS}& \multicolumn{2}{|c|}{RMCA}& \multicolumn{2}{c}{LNS-wPBS}\\
\hline
\multicolumn{1}{c}{$f$} &\multicolumn{1}{c|}{$M$}  & \multicolumn{1}{c}{st} & \multicolumn{1}{c|}{rt}   & \multicolumn{1}{c}{st} & \multicolumn{1}{c|}{rt}  & \multicolumn{1}{c}{st} & \multicolumn{1}{c|}{rt} & \multicolumn{1}{c}{st} & \multicolumn{1}{c}{rt} \\
\thickhline
\multirow{5}{*}{0.2} & 10  & 29.77 & \textbf{28.16} & 27.92& 313.45& \textbf{26.74}& 200.08& 27.87& 316.08 \\
& 20 & 26.70 & \textbf{136.21} & 25.33 &294.94&\textbf{24.28}& 200.74& 25.67& 316.50\\
& 30  & 25.56 & 305.78 & 25.10& \textbf{292.04}&\textbf{23.27}&201.88& 24.69 &298.29\\
& 40  & 25.46 & 415.25 & 24.30&286.58&\textbf{22.62}& \textbf{202.98}&24.58 &291.88 \\
& 50 & 25.05 & 757.40 &24.09&277.72&\textbf{22.37}& \textbf{205.00}&24.39 &292.96 \\
\hline \multicolumn{2}{c|}{Gap}  &  & & -4.2\%&+206.2\% &&&+6.7\%&+50.0\%\\
\hline
\multirow{5}{*}{0.5} & 10  & 109.71 & \textbf{51.23} &116.59 & 400.44 & \textbf{101.62}& 438.57 & 117.44& 382.50 \\
& 20  & 27.99 & \textbf{172.36} & 26.91& 646.19&\textbf{25.44}& 496.28& 27.52 &617.43\\
& 30  & 26.23 & 512.04 & 25.26&667.61&\textbf{23.66}& \textbf{501.24}&25.72 &635.44  \\
& 40  & 25.39 & 1,017.49 & 24.65& 667.25&\textbf{22.73}& \textbf{503.49}&24.84 &657.18 \\
& 50  & 24.94 & 1,736.70 & 23.94& 666.01&\textbf{22.44}& \textbf{508.45}& 24.76 &645.86 \\
\hline \multicolumn{2}{c|}{Gap} & &&   -1.6\%&+178.1\% &&&+10.4\%&+19.1\% \\
\hline
\multirow{5}{*}{1} & 10  &285.75 &\textbf{65.70} & 273.48&448.81&269.76&  464.67& \textbf{266.77} & 419.59\\
& 20  & 75.13 & \textbf{266.76} & 67.21& 880.60&\textbf{59.12}& 851.97&67.20 & 762.55\\
& 30  & 31.41 & \textbf{492.12} & 28.82 & 1,030.93 &\textbf{25.59}& 974.76& 28.05 & 947.47 \\
& 40 & 28.33 & 1,381.56 &25.28 & 1,042.05&\textbf{23.67}& 987.04& 25.62 &\textbf{960.76} \\
& 50  & 27.38 & 3,238.17 & 24.42 & 1,055.77&\textbf{23.01}& 995.00& 25.30& \textbf{958.01}\\
\hline \multicolumn{2}{c|}{Gap} &  & & -8.8\%&+166.1\% &&&+8.0\%&-5.8\%\\
\hline
\multirow{5}{*}{2} 
& 10 & 388.21 & \textbf{81.35} & 361.59& 258.75& 371.27& 231.32& \textbf{356.90} & 229.33\\
& 20  & 162.00 & 424.18 &140.27&477.73&146.81& 444.33&\textbf{140.22}& \textbf{420.59}\\
& 30 & 85.89 & 702.22  & 75.45&749.36& 77.75& 635.75& \textbf{74.30}& \textbf{597.24}\\
& 40  & 57.53  & 1,440.20 &44.55&1,307.76&43.49& 798.31& \textbf{41.90}&\textbf{752.98}\\
& 50 & 41.43 & 2,206.70 &30.46 & 1,249.02&28.88& 927.25& \textbf{28.30}  & \textbf{893.02}\\
\hline \multicolumn{2}{c|}{Gap} & & &  -16.2\%&+36.9\% & & &-3.6\%&-4.3\% \\
\hline
\multirow{5}{*}{5} & 10 & 455.16 &\textbf{85.32} &412.75 &157.27& 435.70& 99.57 &\textbf{408.77} & 109.84\\
& 20 & 229.55 & 422.41 &\textbf{197.28}& 244.39 &209.55& \textbf{184.11}& 197.51 & 187.50\\
& 30  & 147.76 & 1,012.82& 126.41 & 373.18 &132.06& \textbf{268.07}& \textbf{123.95} &272.18\\
& 40  & 108.28 & 1,745.05 & \textbf{90.01}& 627.75&96.81& \textbf{362.65}&91.01&364.22 \\
& 50& 86.90 & 2,686.08 & \textbf{70.31}& 914.49 &74.32& 425.26 & 72.25&\textbf{422.82}\\
\hline \multicolumn{2}{c|}{Gap} & & &  -14.7\%&-30.1\% & & & -5.3\%& +2.7\%\\
\hline
\multirow{5}{*}{10} & 10  & 478.17& 92.96&438.71& 117.76&458.23& \textbf{56.68} &\textbf{431.76} &65.62\\
& 20 &  242.18 & 375.23   &217.33 & 168.63&228.9& \textbf{101.20}& \textbf{215.74} & 110.00\\
& 30  & 165.13 &869.85 & 146.56& 254.68&154.28& \textbf{152.34}& \textbf{144.11} & 163.94 \\
& 40 & 128.39 & 1,723.10 &110.41 & 381.89 &115.04& 208.32&\textbf{109.10}&\textbf{203.33}\\
& 50  & 106.70 & 7,442.20 & \textbf{88.75} &602.85 &94.29& 246.96& 89.33& \textbf{243.87}\\
\hline \multicolumn{2}{c|}{Gap} & & &-12.0\%&-5.7\% & &&-5.7\%&+5.6\%\\
\hline
\multirow{5}{*}{\rotatebox[origin=c]{90}{offline}} & 10  & 501.11& 76.03&  432.28 &58.83 &443.86& 102.07&\textbf{428.39} & \textbf{13.63}\\
& 20 & 263.55 & 374.87  &226.26 &62.16 &230.18& 134.69&\textbf{223.86} &\textbf{25.93} \\
& 30  &187.31& 19,471.88& 159.89&91.85&161.20& 228.56&\textbf{156.04} &\textbf{35.38} \\
& 40 &N/A& N/A&125.67&125.03&126.68& 324.39&\textbf{122.46}& \textbf{60.66}\\
& 50  &N/A& N/A& 107.58&164.69&104.01& 259.34&\textbf{103.12}& \textbf{62.40}\\
\hline \multicolumn{2}{c|}{Gap} & & &-11.5\%&-57.9\% & &&-2.7\%&-81.8\%\\
\thickhline
\end{tabular}
\label{small_warehouse}
}
\end{table}

\textbf{Medium and Large Warehouses.}
TABLE \ref{large_warehouse_fixed_tasks} compares LNS-PBS and LNS-wPBS with the scalable MAPD algorithms HBH-MLA* and RMCA on MAPD instances in \mwarehouse.
Both RMCA and LNS-PBS suffer from scalability issues: they reach the total runtime limit of 1.5h for large numbers of agents. We do not report the results for CENTRAL because its scalability is even worse than that of RMCA and LNS-PBS.
Although we report mixed results of LNS-wPBS and RMCA in the online setting in \swarehouse, here, LNS-wPBS dominates RMCA with respect to both efficiency and effectiveness.
HBH+MLA* turns out to be the most efficient MAPD algorithm in this setting, but LNS-wPBS yields smaller service times (by more than $9\%$) than HBH+MLA* on every MAPD instance, with the maximum gap being $17\%$. 
To further compare the two scalable MAPD algorithms HBH+MLA* and LNS-wPBS, we use a thousand agents with thousands of tasks in \lwarehouse\ and report the reults in TABLE \ref{large_warehouse_fixed_agents}. LNS-wPBS again is slower but yields smaller service times than HBH+MLA* on all MAPD instances.

\begin{table}[t]
\caption{Results on MAPD instances in \mwarehouse\ with $f=50$. ``N/A'' means that the total runtime exceeds 1.5h. 
}
\label{large_warehouse_fixed_tasks}
\centering
\Huge
\scalebox{0.3}{
\begin{tabular}{c|rr|rr|rr|rr}
\thickhline
 & \multicolumn{2}{c|}{HBH+MLA*} & \multicolumn{2}{c|}{RMCA} & \multicolumn{2}{c|}{LNS-wPBS} & \multicolumn{2}{c}{LNS-PBS}\\
\hline
\multicolumn{1}{c|}{$M$}  & \multicolumn{1}{c}{st} & \multicolumn{1}{c|}{rt}    & \multicolumn{1}{c}{st} & \multicolumn{1}{c|}{rt}  & \multicolumn{1}{c}{st} & \multicolumn{1}{c|}{rt} & \multicolumn{1}{c}{st} & \multicolumn{1}{c}{rt} \\
\thickhline
 100 & 362.70&\textbf{1.99}  &329.58 &565.76&\textbf{300.90}& 87.35 & 301.78 & 345.36\\
200 & 207.76&\textbf{6.75} & 192.67&2,072.98&176.81 & 220.28 & \textbf{176.13}& 3,065.95\\
300 & 157.11&\textbf{14.89} &147.42 &4,734.94&139.33 & 465.78& \textbf{137.97} & 8,844.98\\
400 & 136.40&\textbf{32.59} &126.44 &9,906.40&\textbf{123.32}& 806.54 & N/A & N/A\\
500 & 125.42&\textbf{65.79} & N/A & N/A&\textbf{113.78}& 1,385.90 & N/A & N/A\\
\thickhline
\end{tabular}
}
\end{table}

\begin{table}[t]
\caption{Results on MAPD instances  in \lwarehouse\ with $M=1{,}000$ and $f=100$.}
\label{large_warehouse_fixed_agents}
\centering
\Huge
\scalebox{0.3}{
\begin{tabular}{c|rr|rr|r}
\thickhline
 & \multicolumn{2}{c|}{HBH+MLA*} & \multicolumn{2}{c|}{LNS-wPBS}&\\
\hline
tasks & \multicolumn{1}{c}{st} & \multicolumn{1}{c|}{rt} &  \multicolumn{1}{c}{st} & \multicolumn{1}{c|}{rt} & \multicolumn{1}{c}{st gap}\\
\thickhline
1,000 &162.98 & \textbf{373.10}&\textbf{155.00} & 9,288.45 & -4.8\%\\
2,000 & 209.89& \textbf{468.74}&\textbf{193.22} & 7,792.24 & -7.9\%\\
3,000 & 258.74& \textbf{346.44}& \textbf{233.99} & 8,173.96 & -10.5\%\\
4,000 & 307.59&\textbf{400.26}& \textbf{274.54} & 7,730.04 & -10.7\%\\
5,000 & 356.60 &\textbf{487.90}&  \textbf{314.42} & 5,038.37 & -11.8\%\\
\thickhline
\end{tabular}
}
\end{table}

\textbf{Runtime Variance.}
In warehouses, we care not only about the average runtime per timestep but also the runtime variance. We thus plot the runtimes of the algorithms at different timesteps in Fig. \ref{time}. LNS-wPBS is more stable than LNS-PBS and CENTRAL with respect to the runtimes at different timesteps. We also partition the runtime into two parts, namely the task-assignment and path-finding runtimes, in Fig. \ref{time_breakdown}. LNS-PBS and LNS-wPBS always take about 1s to assign tasks, whereas the task-assignment runtime of CENTRAL is less stable. The path-finding runtime of LNS-wPBS is more evenly distributed than those of LNS-PBS and CENTRAL.

\begin{figure}[t]
\centering
\includegraphics[width=0.42\textwidth]{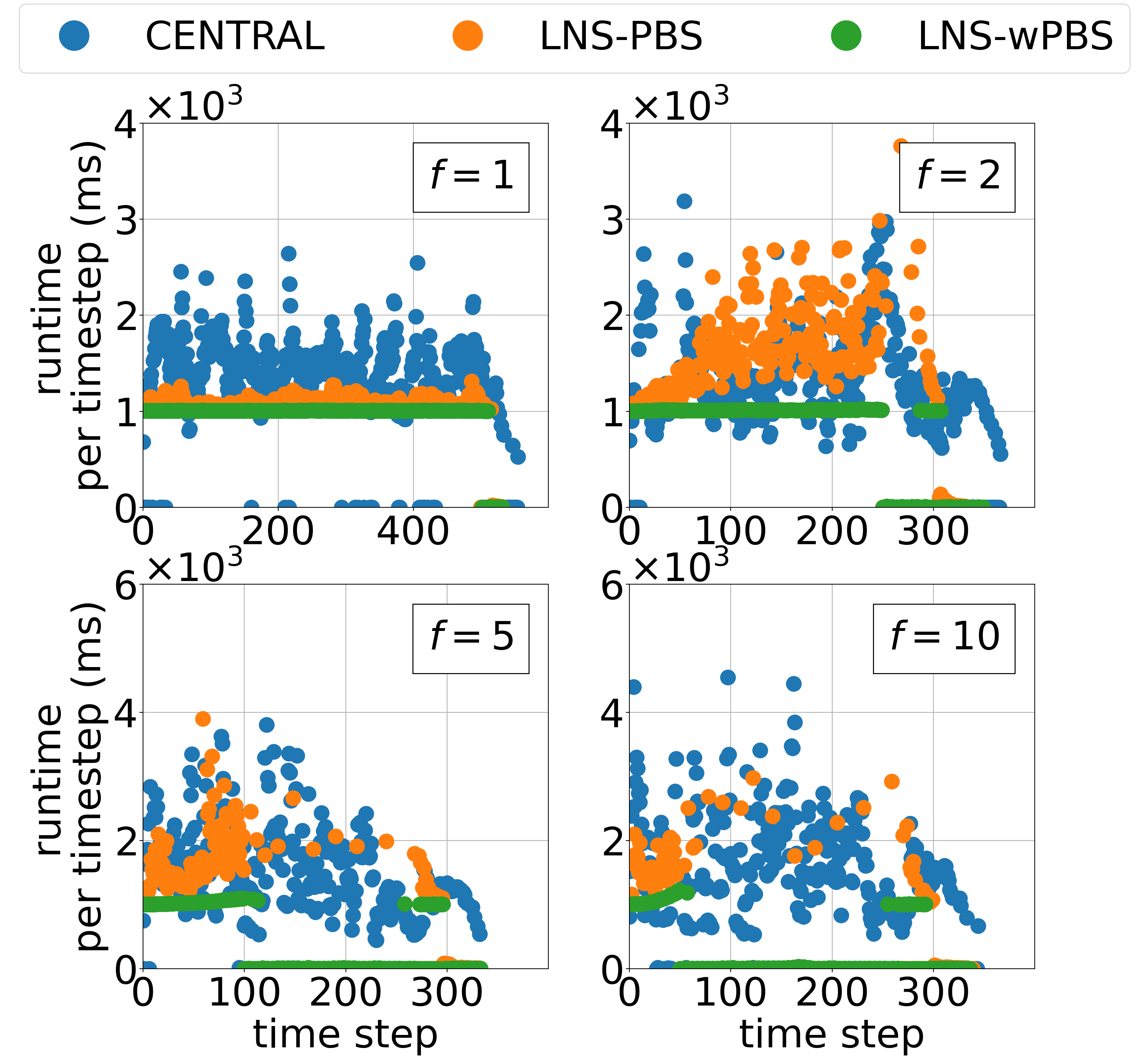}
\caption{Runtime on MAPD instances in \swarehouse\ with $M=40$. When an algorithm is not triggered at a timestep, we plot the dot at 0s.}
\label{time}
\end{figure}

\begin{figure}[t]
\centering
\includegraphics[width=0.48\textwidth]{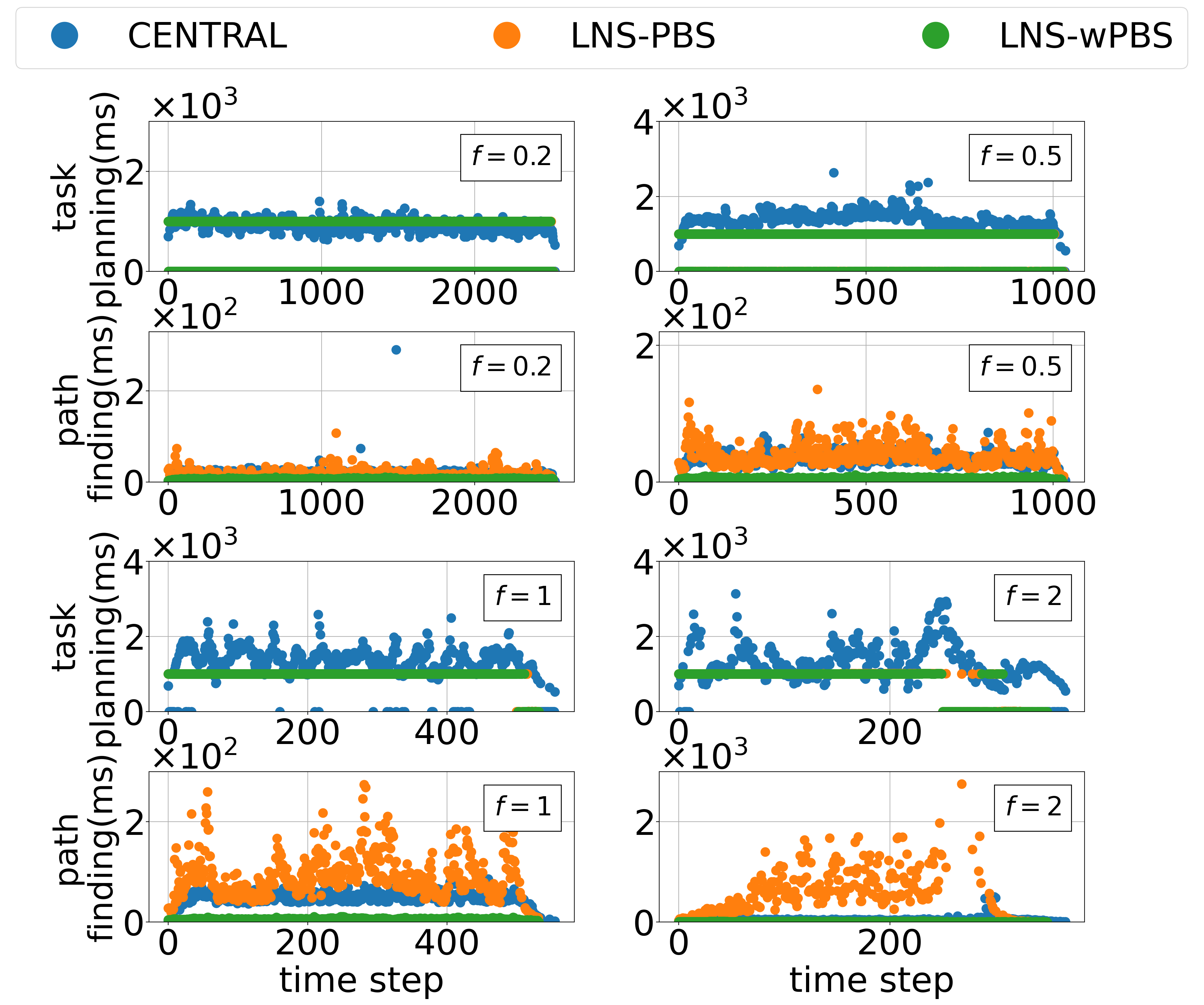}
\caption{Task-assignment and path-finding runtimes  on MAPD instances in \swarehouse\ with $M=40$. When an algorithm is not triggered at a timestep, we plot the dot at 0s. The orange dots in the task-assignment-runtime figures are covered by the green dots since the task-assignment runtimes of LNS-PBS and LNS-wPBS are controlled by the runtime limit of LNS and thus are almost the same.}
\label{time_breakdown}
\end{figure}

\textbf{Task-Assignment Algorithms on MG-MAPD Instances.} TABLE \ref{tp_heuristics} compares LNS-wPBS using different task-assignment algorithms, namely LNS, that is our original task-assignment algorithm, greedy-LNS, that modifies our LNS by using the greedy heuristic from \cite{RopkeTS06, ChenIEEE21} to construct the initial task assignment, and Hungarian, that uses our Hungarian-based insertion to find a task assignment (but does not improve it via LNS).
In most cases, LNS yields the smallest service time, which indicates that our Hungarian-based insertion is more effective than the greedy heuristic from \cite{RopkeTS06, ChenIEEE21} for constructing the initial task assignment, and our LNS improves the initial task assignment, even though the runtime limit of LNS is only 1s. 

\begin{table}[t!]
\caption{Service time on MG-MAPD instances in \swarehouse. The gap is measured between Hungarian and LNS.}
\centering
\Huge
\scalebox{0.3}{
\begin{tabular}{c|c|r|r|r|r}
\thickhline
$f$ & $M$ & Hungarian & Greedy-LNS & LNS & Gap\\
\thickhline
\multirow{5}{*}{2} 
& 10 & 623.87 & \textbf{590.54} & 597.05 &-4.2\% \\
& 20  & 283.83& 277.65& \textbf{269.87} & -4.9\%\\
& 30 & 173.15 & 177.21& \textbf{170.07} & -1.7\% \\
& 40  & \textbf{122.21} & 129.57 & 122.24 & +0.0\% \\
& 50 & 98.10 & 99.58& \textbf{95.94} &  -2.6\%\\
\hline
\multirow{5}{*}{5} & 10 & 683.50 & \textbf{637.66} &  645.21 &-5.6\% \\
& 20 & 332.23& 327.57& \textbf{318.55} & -4.1\%\\
& 30  & 224.21 & 221.70 & \textbf{218.27} &-2.6\%\\
& 40  & 168.11 & 173.47 & \textbf{164.80} &  -1.9\%\\
& 50& 141.61 & 141.95 & \textbf{139.86} &  -1.2\%\\
\hline
\multirow{5}{*}{10} & 10  & 683.60 & 654.26 & \textbf{650.63} & -4.8\%\\
& 20 &  343.88 & 341.56 & \textbf{ 336.97} & -2.0\%\\
& 30  &239.39 & 239.20 & \textbf{ 229.76} & -4.0\%\\
& 40 & 184.94 & 187.55 & \textbf{178.19} & -3.6\%\\
& 50  & 157.01 & 156.07 & \textbf{149.91}& -4.5\%\\
\thickhline
\end{tabular}
}
\label{tp_heuristics}
\end{table}

\textbf{Look-Ahead Horizons on MG-MAPD Instances.} TABLE \ref{look_ahead_horizon} shows that knowing future tasks helps planning to obtain smaller service times, but the benefit diminishes for longer look-ahead horizons. For the MG-MAPD instances that we test, extending the look-ahead horizon from 0 to 5 always reduces the service times, but extending it from 5 to 10 increases the service times for small numbers of agents. We suspect that this is so because we sometimes need to sacrifice the service times of the first few tasks in order to optimize entire task sequences. Furthermore, the task sequences for MG-MAPD instances with large look-ahead horizons and small numbers of agents can be very long and change frequently, meaning that the agents never execute them completely as planned.

\begin{table}[t!]
\caption{Service time on MG-MAPD instances in \swarehouse\ with $f=2$. ``LA$x$'' means that LNS-wPBS looks $x$ batches of tasks ahead. The numbers in parentheses are the gaps measured with respect to LA0.}
\centering
\Huge
\scalebox{0.3}{
\begin{tabular}{c|c|c|c|c}
\thickhline
$M$ & LA0 & LA1 & LA5 & LA10\\
\thickhline
10 & 58.93 &54.07 (-8.2\%)&\textbf{48.89} (-17.0\%)&65.25 (+10.6\%)\\
20 & 46.52 & 41.85 (-10.0\%)& \textbf{37.94} (-18.4\%)& 38.12 (-18.0\%)\\
30 & 45.36 & 41.59 (-8.3\%)& \textbf{38.30} (-15.5\%)&38.41 (-15.3\%)\\
40 & 45.17 & 41.99 (-7.0\%)& \textbf{38.80} (-14.1\%)&38.84 (-14.0\%)\\
50 & 45.00 & 41.57 (-7.6\%)& \textbf{39.11} (-13.0\%)&39.31 (-12.6\%)\\
\thickhline
\end{tabular}
}
\label{look_ahead_horizon}
\end{table}
\section{CONCLUSIONS}
In this work, we proposed two variants of a decoupled MAPD algorithm that assigns task sequences to agents and plans collision-free paths for them through the corresponding sequences of goal locations. The first variant, LNS-PBS, is complete for well-formed MAPD instances, and the second variant, LNS-wPBS, is more efficient and stable. 
Empirically, both of them produce solutions with smaller service times than the state-of-the-art MAPD algorithms, and LNS-wPBS can scale to a thousand agents with thousands of tasks in a large warehouse. As a further contribution, our algorithm extends to Multi-Goal MAPD, where tasks have different numbers of goal locations, and is capable of handling semi-online settings where we know the tasks in the near future.

In future work, we intend to improve our algorithms by letting them assign tasks based on actual path costs. We also intend to use a more realistic robotic simulator to demonstrate the potential of our algorithms to run on actual robots.


\bibliographystyle{./bibliography/IEEEtran} 
\bibliography{./bibliography/IEEEexample}

\begin{thebibliography}{10}
\providecommand{\url}[1]{#1}
\csname url@rmstyle\endcsname
\providecommand{\newblock}{\relax}
\providecommand{\bibinfo}[2]{#2}
\providecommand\BIBentrySTDinterwordspacing{\spaceskip=0pt\relax}
\providecommand\BIBentryALTinterwordstretchfactor{4}
\providecommand\BIBentryALTinterwordspacing{\spaceskip=\fontdimen2\font plus
\BIBentryALTinterwordstretchfactor\fontdimen3\font minus
  \fontdimen4\font\relax}
\providecommand\BIBforeignlanguage[2]{{%
\expandafter\ifx\csname l@#1\endcsname\relax
\typeout{** WARNING: IEEEtran.bst: No hyphenation pattern has been}%
\typeout{** loaded for the language `#1'. Using the pattern for}%
\typeout{** the default language instead.}%
\else
\language=\csname l@#1\endcsname
\fi
#2}}

\bibitem{MaAAMAS17}
H.~Ma, J.~Li, T.~K.~S. Kumar, and S.~Koenig, ``Lifelong multi-agent path
  finding for online pickup and delivery tasks,'' in \emph{International
  Conference on Autonomous Agents and Multiagent Systems}, 2017, pp. 837--845.

\bibitem{SternSOCS19}
R.~Stern, N.~Sturtevant, A.~Felner, S.~Koenig, H.~Ma, T.~Walker, J.~Li,
  D.~Atzmon, L.~Cohen, T.~K.~S. Kumar, E.~Boyarski, and R.~Bartak,
  ``Multi-agent pathfinding: Definitions, variants, and benchmarks,'' in
  \emph{International Symposium on Combinatorial Search}, 2019, pp. 151--159.

\bibitem{GrenouilleauICAPS21}
F.~Grenouilleau, W.-J.~v. Hoeve, and J.~N. Hooker, ``A multi-label {A}*
  algorithm for multi-agent pathfinding,'' in \emph{International Conference on
  Automated Planning and Scheduling}, 2021, pp. 181--185.

\bibitem{LiuAAMAS19}
M.~Liu, H.~Ma, J.~Li, and S.~Koenig, ``Task and path planning for multi-agent
  pickup and delivery,'' in \emph{International Conference on Autonomous Agents
  and Multiagent Systems}, 2019, pp. 2253--2255.

\bibitem{ChenIEEE21}
Z.~Chen, J.~Alonso-Mora, X.~Bai, D.~D. Harabor, and P.~J. Stuckey, ``Integrated
  task assignment and path planning for capacitated multi-agent pickup and
  delivery,'' \emph{IEEE Robotics and Automation Letters}, vol.~6, no.~3, pp.
  5816--5823, 2021.

\bibitem{SurynekAAAI21}
P.~Surynek, ``Multi-goal multi-agent path finding via decoupled and integrated
  goal vertex ordering,'' in \emph{{AAAI} Conference on Artificial
  Intelligence}, 2021, pp. 12\,409--12\,417.

\bibitem{LiAAAI21b}
J.~Li, A.~Tinka, S.~Kiesel, J.~W. Durham, T.~K.~S. Kumar, and S.~Koenig,
  ``Lifelong multi-agent path finding in large-scale warehouses,'' in
  \emph{AAAI Conference on Artificial Intelligence}, 2021, pp.
  11\,272--11\,281.

\bibitem{GerkeyIJRR04}
B.~P. Gerkey and M.~J. Matari\'c, ``A formal analysis and taxonomy of task
  allocation in multi-robot systems,'' \emph{International Journal of Robotics
  Research}, pp. 939--954, 2004.

\bibitem{KorsahIJRR13}
G.~A. Korsah, A.~Stentz, and M.~B. Dias, ``A comprehensive taxonomy for
  multi-robot task allocation,'' \emph{International Journal of Robotics
  Research}, pp. 1495--1512, 2013.

\bibitem{Kuhn55thehungarian}
H.~W. Kuhn and B.~Yaw, ``The {H}ungarian method for the assignment problem,''
  \emph{Naval Research Logistics Quarterly}, pp. 83--97, 1955.

\bibitem{Shaw97anew}
P.~Shaw, ``A new local search algorithm providing high quality solutions to
  vehicle routing problems,'' \emph{APES Group, Dept of Computer Science,
  University of Strathclyde}, 1997.

\bibitem{SharonAI15}
G.~Sharon, R.~Stern, A.~Felner, and N.~R. Sturtevant, ``Conflict-based search
  for optimal multi-agent pathfinding,'' in \emph{AAAI Conference on Artificial
  Intelligence}, 2012, pp. 563--569.

\bibitem{BoyarskiIJCAI15}
E.~Boyarski, A.~Felner, R.~Stern, G.~Sharon, D.~Tolpin, O.~Betzalel, and
  E.~Shimony, ``{ICBS}: Improved conflict-based search algorithm for
  multi-agent pathfinding,'' in \emph{International Joint Conference on
  Artificial Intelligence}, 2015, pp. 740--746.

\bibitem{ErdmannAl87}
M.~Erdmann and T.~Lozano-Perez, ``On multiple moving objects,'' in \emph{IEEE
  International Conference on Robotics and Automation}, 1986, pp. 1419--1424.

\bibitem{MaAAAI19a}
H.~Ma, D.~D. Harabor, P.~J. Stuckey, J.~Li, and S.~Koenig, ``Searching with
  consistent prioritization for multi-agent path finding,'' in \emph{{AAAI}
  Conference on Artificial Intelligence}, 2019, pp. 7643--7650.

\bibitem{Helsgaun2017}
K.~Helsgaun, ``An extension of the {L}in-{K}ernighan-{H}elsgaun {TSP} solver
  for constrained traveling salesman and vehicle routing problems,''
  \emph{Roskilde: Roskilde University}, 2017.

\bibitem{RopkeTS06}
S.~Ropke and D.~Pisinger, ``An adaptive large neighborhood search heuristic for
  the pickup and delivery problem with time windows,'' \emph{Transportation
  Science}, pp. 455--472, 2006.

\end{thebibliography}
\end{document}